\documentclass{article}

\usepackage[utf8]{inputenc} 
\usepackage[T1]{fontenc}    
\usepackage[english]{babel}
\usepackage{hyperref}       
\usepackage{url}            
\usepackage{booktabs}       
\usepackage{amsfonts}       
\usepackage{nicefrac}       
\usepackage{microtype}      
\usepackage{mathtools}
\usepackage{amsmath,amssymb}
\usepackage{amsthm}
\usepackage{thmtools,thm-restate}
\usepackage{float}
\usepackage{array}
\usepackage{graphicx}
\usepackage{caption}
\usepackage{subcaption}
\usepackage[title]{appendix}
\usepackage{tikz}
\usetikzlibrary{positioning,angles,quotes}

\usepackage{hyperref}

\usepackage[accepted]{icml2019}

\icmltitlerunning{Approximate Exploration through State Abstraction}

\begin{document}

\twocolumn[
\icmltitle{Approximate Exploration through State Abstraction}



\icmlsetsymbol{equal}{*}

\begin{icmlauthorlist}
\icmlauthor{Adrien Ali Ta\"{i}ga}{mila}
\icmlauthor{Aaron Courville}{mila,cifar}
\icmlauthor{Marc G. Bellemare}{google,cifar}
\end{icmlauthorlist}

\icmlaffiliation{mila}{MILA, Universit\'{e} de Montr\'{e}al}
\icmlaffiliation{google}{Google Brain}
\icmlaffiliation{cifar}{CIFAR Fellow}

\icmlcorrespondingauthor{Adrien Ali Ta\"{i}ga}{adrien.ali.taiga@umontreal.ca}

\icmlkeywords{Pseudo-counts, State abstraction, Exploration, Reinforcement learning}

\vskip 0.3in
]



\printAffiliationsAndNotice{}  

\newtheorem{theorem}{Theorem}
\newtheorem{corollary}{Corollary}[theorem]
\newtheorem{lemma}{Lemma}
\newtheorem{assumption}{Assumption}
\newtheorem{proposition}{Proposition}
\newtheorem{example}{Example}
\newtheorem*{remark}{Remark}

\theoremstyle{definition}
\newtheorem{definition}{Definition}

\def \cS {\mathcal{S}}
\def \Sg {\mathcal{S}_G}
\def \Sa {\mathcal{S}_A}
\def \Rg {\mathcal{R}_G}
\def \Ra {\mathcal{R}_A}
\def \Tg {\mathcal{T}_G}
\def \Ta {\mathcal{T}_A}
\def \hN {\hat N_n}
\def \tpi {\tilde{\pi}}
\def \tN {\tilde N_n}
\def \hn {\hat N_}
\def \tn {\tilde N}
\def \hNsa {\hat N_n (s, a)}
\def \tNsa {\tilde N_n (s, a)}
\def \bs {\bar{s}}
\def \gs {|G(s)|}
\def \hNa {\hat N^{\text A}_n}
\def \Na {N^{\text A}_n}
\def \Nasa {N^{\text A}_n (\bar{s}, a)}
\def \hNasa {\hat N^{\text A}_n (\bar{s}, a)}
\def \hNapsa {\hat N'^{\text A}_n  (\bar{s}, a)}
\def \hna {\hat n^{\text A}}
\def \tn {\tilde{n}}
\def \ra   {\rho^{\text A}_n}
\def \rsa   {\rho_n (s, a)}
\def \rpsa   {\rho'_n (s, a)}
\def \rppsa   {\rho^{(2)}_n (s, a)}
\def \rasa   {\rho^{\text A}_n (\bar{s}, a)}
\def \rapsa   {\rho'^{\text A}_n  (\bar{s}, a)}
\def \expect {\mathop{\mathbb{E}}}

\begin{abstract}
Although exploration in reinforcement learning is well understood from a theoretical point of view, provably correct methods remain impractical. 
In this paper we study the interplay between exploration and approximation, what we call \emph{approximate exploration}.
Our main goal is to further our theoretical understanding of pseudo-count based exploration bonuses \citep{bellemare2016unifying}, a practical exploration scheme based on density modelling. 
As a warm-up, we quantify the performance of an exploration algorithm, MBIE-EB \citep{strehl2008analysis}, when explicitly combined with state aggregation. 
This allows us to confirm that, as might be expected, approximation allows the agent to trade off between learning speed and quality of the learned policy. 
Next, we show how a given density model can be related to an abstraction and that the corresponding pseudo-count bonus can act as a substitute in MBIE-EB combined with this abstraction, but may lead to either under- or over-exploration. 
Then, we show that a given density model also defines an implicit abstraction, and find a surprising mismatch between pseudo-counts derived either implicitly or explicitly. Finally we derive a new pseudo-count bonus alleviating this issue.
\end{abstract}

\section{Introduction}
\label{submission}

In reinforcement learning (RL), an agent's goal is to maximize the expected sum of  future rewards obtained through interactions with a unknown environment. In doing so, the agent must balance \emph{exploration} -- acting to improve its knowledge of the environment -- and \emph{exploitation}: acting to maximize rewards according to its current knowledge. In the tabular setting, where each state can be modelled in isolation, near-optimal exploration is by now well understood and a number of algorithms provide finite time guarantees \citep{brafman2002r,strehl2008analysis,jaksch2010near,szita10modelbased,osband2014near,azar17minimax}. 
To guarantee near-optimality, however, the sample complexity of theoretically-motivated exploration algorithms must scale at least linearly with the number of states in the environment \citep{azar12sample}.

Yet, recent empirical successes have shown that practical exploration is not hopeless \citep{bellemare2016unifying,count-based,pathak17curiositydriven,plappert18parameter,fortunato18noisy,burda2018exploration}. In this paper we use the term \emph{approximate exploration} to describe algorithms which sacrifice near-optimality in order to explore more quickly. A desirable characteristic of these algorithms is fast convergence to a reasonable policy; near-optimality may be achieved when the environment is ``nice enough''. 


Our specific aim is to gain new theoretical understanding of the pseudo-count method, introduced by \citet{bellemare2016unifying} as a means of estimating visit counts in non-tabular settings, and how this pertains to approximate exploration.
Our study revolves around the MBIE-EB algorithm \citep[Model-based Interval Estimation with Exploration Bonuses;][]{strehl2008analysis} as a simple illustration of the general ``optimism in the face of uncertainty'' principle in an approximate exploration setting; MBIE-EB drives exploration by augmenting the empirical reward function with a count-based exploration bonus, which can be either derived from real counts or pseudo-counts.

As a warm-up, we construct an \emph{explicitly approximate} exploration algorithm by applying MBIE-EB to an abstract environment based on state abstraction \citep{li2006towards,abel2016near}. In this setting we derive performance bounds that simultaneously depend on the quality and size of the aggregation: by taking a finer or coarser aggregation, one can trade off exploration speed and accuracy. We then relate pseudo-counts to these aggregations and show how using pseudo-counts within MBIE-EB can lead to \emph{under-exploration} (failing to achieve theoretical guarantees) or \emph{over-exploration} (using an excessive number of samples to do so). Additionally, we quantify the magnitude of both phenomena.

Finally, we show that using pseudo-counts for exploration in the wild, as has been done in practice, produces \emph{implicitly approximate} exploration. Specifically, under certain assumptions on the density model generating the pseudo-counts, these behave approximately as if derived from a particular abstraction. This is in general problematic, as in pathological cases this prohibits any kind of theoretical guarantees. As an interesting corollary, we find a surprising mismatch between the behaviour of these pseudo-counts and what might be expected given the abstraction they implicitly define.

\section{Background and Notations}

We consider a Markov decision process (MDP) represented by a 5-tuple $\langle \mathcal{S}, \mathcal{A}, \mathcal{T}, \mathcal{R}, \gamma \rangle$ with $\mathcal{S}$ a finite state space, $\mathcal{A}$ a finite set of actions, $\mathcal{T}$ a transition probability distribution, $\mathcal{R}: \mathcal{S} \times \mathcal{A} \mapsto [0, 1]$ a reward function, and $\gamma \in [0, 1)$ the discount factor.
The goal of reinforcement learning is to find the optimal policy $\pi^*: \mathcal{S} \mapsto \mathcal{A}$ which maximizes the expected discounted sum of future rewards.
For any policy $\pi$, the $Q$-value of any state-action pair $(s,a)$ describes the expected discounted return after taking action $a$ in state $s$, then following $\pi$ and can be obtained using the Bellman equation
\begin{equation*}
 Q^{\pi} (s, a) = \mathcal{R} (s,a) + \gamma \mathbb{E}_{\mathcal{T}(s' | s, a)} Q^{\pi} (s', \pi(s')).
\end{equation*}
We also introduce $V^{\pi}(s) = Q^{\pi} (s, \pi(s))$ which is the expected discounted return when starting in $s$ and following $\pi$. The $Q$-value of the optimal policy $Q^*$ verifies the optimal Bellman equation
\begin{equation*}
    Q^* (s, a) = \mathcal{R} (s,a) + \gamma \mathbb{E}_{\mathcal{T}(s' | s, a)} \max_{a'} Q^* (s', a').
\end{equation*}
We also write $V^*(s) =  \max_a Q^* (s, a)$. Furthermore we assume without loss of generality that rewards are bounded between 0 and 1, and we denote by \textsc{Qmax} $= 1/(1 - \gamma)$ the maximum $Q$ value.

\subsection{Approximate state abstraction}

We use here the notation from \citet{abel2016near}. An abstraction is defined as a mapping from the state space of a ground MDP, $M_G$, to that of an abstract MDP, $M_A$, using a state aggregation function $\phi$. We will write $\langle \Sg, \mathcal{A}, \Tg, \Rg, \gamma \rangle$ and $\langle \Sa \mathcal{A}, \Ta, \Ra, \gamma \rangle$ for the ground and abstract MDPs respectively. The abstract state space is defined as the image of the ground state space by the mapping $\phi \colon \Sg \rightarrow \Sa$
\begin{equation}
    \mathcal{S}_A = \Big\{ \phi (s) | s \in \mathcal{S}_G  \Big\}.
\end{equation}
We will write $\bar{s} = \phi (s)$ for the abstract state associated to a state $s$ in the ground space.
We define
\begin{align}
G(s) &= \{ g \in \mathcal{S}_G  | \phi(g) = \phi(s) \} \text{ if } s \in \mathcal{S}_G, \\
G(\bar s) &= \{ g \in \mathcal{S}_G  | \phi(g) = \bar{s} \} \text{ if } \bar s \in \mathcal{S}_A .
\end{align}
Let $\omega$ be a weighting such that for all $s \in \mathcal{S}_G$, $0 \le \omega(s) \le 1$ and $\sum\nolimits_{s' \in G(s)} \omega (s') = 1$. We define the abstract rewards and transition functions as the following convex combinations
\begin{align*}
    \forall s, \bar s \in \mathcal{S}_A, \, \mathcal{R}_A (\bar s, a) &\coloneqq \sum_{g \in G(\bar s)}^{} \omega (g) \Rg (g,a), \\
    \mathcal{T}_A (\bar{s}, a, \bar{s}') &\coloneqq  \sum_{g \in G(\bar{s})} \sum_{g' \in G(\bar{s}')} \omega (g) \mathcal{T}_G (g,a,g') .
\end{align*}


Prior work such as \citet{li2006towards} has been mostly focused on exact abstraction in MDPs. While interesting, this notion is usually too restrictive and we will instead consider approximate abstractions \citep{abel2016near}
\begin{definition} \label{defn:approx_abstraction}
Let $\eta > 0$ and $f_\eta \colon \mathcal{S}_G \times \mathcal{A} \rightarrow \mathbb{R}$, $\phi_\eta$ defines an approximate state abstraction as follows
\begin{equation*}
\begin{aligned}
\forall s_1, s_2 \in \Sg, \phi_{\eta}(s_1) = \phi_{\eta} (s_2) \to |f_\eta (s_1) - f_\eta (s_2)| \le \eta .
\end{aligned}
\end{equation*}
\end{definition}

Throughout this paper we will illustrate our results with the \emph{model similarity abstraction}, also known as approximate homomorphism \citep{ravindran2004approximate} or $\epsilon$-equivalent MDP \citep{even2003approximate}
\begin{example} \label{defn:model_similiarity}
Given $\eta > 0$, we let $\phi_\eta$ be such that:
\begin{equation*}
\begin{aligned}
\forall &s_1, s_2 \in \mathcal{S}_G, \, \, \phi_{\eta}(s_1) = \phi_{\eta} (s_2) \to \\ 
&\forall a \in \mathcal{A}, \:  | \Rg (s_1, a) - \Rg (s_2, a)  | \le \eta \, \text{ and } \\
&\forall \bar{s}' \in \mathcal{S}_A, \, \Big | \sum_{s' \in G(\bar{s}')}^{} \big[ \mathcal{T}_G (s_1, a, s') - \mathcal{T}_G (s_2, a, s') \big] \Big| \le \eta.
\end{aligned}
\end{equation*}
\end{example}
Where co-aggregated states have close rewards and transition probabilities to other aggregations. 

Let $\pi_A^{*} : \mathcal{S}_A \to \mathcal{A}$ and $\pi_G^{*} : \mathcal{S}_G \to \mathcal{A}$ be the optimal policies in the abstract and ground MDPs. We are interested in the quality of the policy learned in the abstraction when applied in the ground MDP. For a state $s \in \mathcal{S}_G$ and a state aggregation function $\phi$, we define $\pi_{GA}$ such that
\begin{equation*}
\pi_{GA} (s) = \pi_A^{*} (\phi (s)) .
\end{equation*}
We will also write $Q_G$ and $V_G$ (resp. $Q_A$ and $V_A$) the optimal Q-value and value functions in the ground (resp. abstract) MDP.

\subsection{Optimal exploration and model-based interval estimation exploration bonus (MBIE-EB)}
Exploration efficiency in reinforcement learning can be evaluated using the notions of sample complexity and PAC-MDP introduced by \citet{kakade2003sample}. We now briefly introduce both of these.
\theoremstyle{definition}
\begin{definition}
Define the \emph{sample complexity} $T$ of an algorithm $\textbf{A}$ to be the number of time steps where its policy $\textbf{A}_t$ at state $s$ is not $\epsilon$-optimal: $V^{\textbf{A}_t}(s) < V^{*}(s) - \epsilon$. An algorithm \textbf{A} is said to be PAC-MDP (\emph{"Probably Approximately Correct for MDPs"}) if given a fixed $\epsilon > 0$ and $0 < \delta < 1$ its sample complexity $T$ is less than a polynomial function in the parameters $(\left|{\mathcal{S}}\right|, \left|{\mathcal{A}}\right|, 1/\epsilon, 1/\delta, 1/(1 - \gamma))$ with probability at least $1 - \delta$.
\end{definition}
We focus on MBIE-EB as a simple algorithm based on the state-action visit count, noting that more refined algorithms now exist with better sample guarantees \citep[e.g.][]{azar17minimax,dunn17unifying} and that our analysis would extend easily to other algorithms based on state-action visit count. MBIE-EB learns the optimal policy by solving an empirical MDP based upon estimates of rewards and transitions and augments rewards with an exploration bonus
\begin{equation}
\label{eq:mbieeb}
V (s) = \max_{a \in \mathcal{A}} \Bigg[\hat{\mathcal{R}} (s, a) + \gamma \expect\limits_{\mathclap{\hat{\mathcal{T}} (s'  | s, a)}} V (s') + \frac{\beta}{\sqrt{N_n (s ,a)}} \Bigg].
\end{equation}

\begin{theorem}[\citet{strehl2008analysis}] \label{th:mbieeb}
Let $\epsilon, \delta > 0$ and consider an MDP $\text{M} = \langle \mathcal{S}, \mathcal{A}, \mathcal{T}, \mathcal{R}, \gamma \rangle$. Let $\textbf{A}_{t}$ denote MBIE-EB executed on M with parameter $\beta = (1 / (1 - \gamma)) \sqrt{\ln (2 | \mathcal{S} | | A | m / \delta ) / 2}$, with $m$ an algorithmic constant, and let $s_t$ denote the state at time t. With probability at least $1 - \delta$, $V^{\text{\textbf{A}}_t}(s_t) \geq V^{*} (s_t) - \epsilon$ will hold for all but $T$ time steps, with
\begin{equation}\label{eqn:pac_bound}
T = \tilde O \Big( \frac{| \mathcal{S} |^2 | \mathcal{A} |}{\epsilon^{3} (1 - \gamma)^{6}} \Big).
\end{equation}
\end{theorem}

\subsection{Pseudo-counts}
Pseudo-counts have been proposed as a way to estimate counts using a density model $\rho$ over state-action pairs.
Given $s_{1:n} \in \mathcal{S}_G^n$ a sequence of states and $a_{1:n} \in \mathcal{A}^n$ a sequence of actions, we write $\rho_n (s, a) \coloneqq \rho (s, a; s_{1:n}, a_{1:n})$ the probability assigned to $(s,a)$ after training on $s_{1:n}, a_{1:n}$. After training on $(s, a)$, we write the new probability assigned as $\rho_n '(s, a) \coloneqq \rho (s, a; s_{1:n}s, a_{1:n}a)$, where $s_{1:n}s$ denotes the concatenation of sequences $s_{1:n}$ and $s$. We require the model to be \emph{learning-positive} i.e $\rho_n '(s, a) \ge \rho_n (s, a)$ and define the pseudo-count
\begin{equation*}
    \hN (s, a) = \frac{\rho_n (s, a) (1 - \rho_n '(s, a))}{\rho^{'}_n (s, a) - \rho_n (s, a)}.
\end{equation*}
Which is derived from requiring a one unit increase of the pseudo-count after observing $(s, a)$:
\begin{equation*}
    \rho_n(s, a) = \frac{\hN (s, a)}{\hat{n}}, \quad \quad  \rho_n '(s, a) = \frac{\hN (s, a)+1}{\hat{n}+1}.
\end{equation*}
Where $\hat{n}$ is the pseudo-count total. We also define the empirical density derived from the state-action visit count $N_n(s, a)$
\begin{equation*}
    \mu_n (s, a) \coloneqq \mu (s, a; s_{1:n}, a_{1:n}) \coloneqq \frac{N_n(s, a)}{n}.
\end{equation*}
Notice that when $\rho_n = \mu_n$ the pseudo-count is consistent and recovers $\hN (s, a) = N_n (s, a)$.
We will also be interested in exploration in abstractions, and to that end define the count of an aggregation A
\begin{align*}
\forall \bar{s} \in \mathcal{S}_A, \quad N_n^{\text{A}} (\bar{s}, a) &= \sum_{s \in G(\bar{s})} N_n (s, a).
\end{align*}

\section{Explicitly approximate exploration}\label{sec:explicit_abstraction}

While PAC-MDP algorithms provide guarantees that the agent will act close to optimally with high probability, their sample complexity must increase at least linearly with the size of the state space \citep{azar12sample}. In practice, algorithms are often given small budgets and may not be able to discover the optimal policy within this time. Dealing with smaller sample budgets is exactly the motivation behind approximate exploration methods such as \citet{bellemare2016unifying}'s, which we will analyze in greater detail in later sections.

When faced with a small budget, it might be appealing to sacrifice near-optimality for sample complexity. One way to do so is to derive the exploratory policy from a given abstraction. We call this process \emph{ explicitly approximate} exploration. As we now show, using a model similarity abstraction is a particularly appealing scheme for explicitly approximate exploration.
MBIE-EB applied to the abstract MDP $M_A$ solves the following equation
\begin{equation}
\label{eq:explicit_empi}
V (\bar{s}) = \max_{a \in \mathcal{A}} \Bigg[\hat{\mathcal{R}}_A (\bar{s}, a) + \gamma \expect\limits_{\mathclap{\hat{\mathcal{T}}_A (\bar{s}' \vert \bar{s}, a)}} V (\bar{s}') + \frac{\beta}{\sqrt{N^{\text{A}}_n (\bar{s} ,a)}} \Bigg].
\end{equation}
To provide a setting facilitating exploration we first require the abstraction to have sub-optimality bounded in $\eta$:
\begin{definition}\label{defn:subopt}
An abstraction $\phi_\eta$ is said to have sub-optimality bounded in $\eta$ if there exists a function $g$, monotonically increasing in $\eta$, with $g(0) = 0$ such that
\begin{equation*}
\forall s \in \mathcal{S}_G, \, |V_G (s) - V_A (s)| \le g(\eta).
\end{equation*}
And for any policy $\tpi_A: \Sa \rightarrow [0,1]^\mathcal{A}$.
\begin{gather*}
\forall s \in \mathcal{S}_G, \, |V_G^{\tpi_{GA}} (s) - V_A^{\tpi_A} (s)| \le g(\eta)  
\end{gather*}
\end{definition}


Definition \ref{defn:subopt} requires that for $\eta$ small enough we can recover a near-optimal policy using $M_A$ while working with a state space that can be significantly smaller than the ground state space. This property is verified by several abstractions studied by \citet{abel2016near}.

Though when the abstraction is only approximate, learning the optimal policy of the abstract MDP does not imply recovering the optimal policy of the ground MDP.
\begin{restatable}{proposition}{prop}
\label{propopo}
For any $0 < \eta < 1$ , there exists $\epsilon > 0$ and an MDP which defines a model similarity abstraction of parameter $\eta$ over its abstract space such that $\pi_{GA}$ is not $\epsilon$-optimal.
\end{restatable}

We can nevertheless benefit from exploring using the abstract MDP. Combining Theorem \ref{th:mbieeb} and Definition \ref{defn:subopt}:
\begin{restatable}{proposition}{fastmbie}
\label{prop:fast_mbie}
Given an approximate abstraction $\phi_\eta$ with sub-optimality bounded in $\eta$, let $0 < \delta < 1$, $\tilde{\pi}_A$ the (time-dependent) policy obtained while running MBIE-EB in the abstract MDP with $\epsilon = g(\eta)$ and the derived MBIE policy $\tilde{\pi}_{GA} (s) = \tilde{\pi}_{A}(\bs)$, then with probability $1 - \delta$, the following bound holds for all but $T$ time steps:
\begin{gather*}
V_G (s) - V^{\tilde{\pi}_{GA}}_G (s) \le 3 g(\eta) \text{ with } T = \tilde{O} \Big ( \frac{| \mathcal{S}_A |^{2} | \mathcal{A} |}{\epsilon^{3} (1 - \gamma)^{6}} \Big ).
\end{gather*}
\end{restatable}
Proposition \ref{prop:fast_mbie} informs us that even though we cannot guarantee $\epsilon$-optimality for arbitrary $\epsilon > 0$, the abstraction may explore significantly faster, with a sample complexity that depends on $|\cS_A|$ rather than $|\cS_G|$.

We note that a related result is given by \citet{li2009unifying}, where they extended Delayed Q-learning \citep{strehl2006pac} to approximate $Q^*$-irrelevant abstractions \citep{li2006towards}. Our result differs from theirs as it makes explicit the trade off between near-optimality and low sample complexity.

\section{Under- and over-exploration with pseudo-counts}
Results from \citet{count-based} suggest that the choice of density model plays a crucial role in the exploratory value of pseudo-counts bonuses. Thus far, the only theoretical guarantee concerning pseudo-counts is given by Theorem 2 from \citet{bellemare2016unifying} and quantifies the asymptotic behaviour of pseudo-counts derived from a density model. 
We provide here an analysis of the finite time behaviour of pseudo-counts which is then used to give PAC-MDP guarantees. We show that for any given abstraction $A$ a density model can be learned over the abstraction then used to approximate the bonus of Equation \ref{eq:explicit_empi}.

\begin{definition}
Let $(\rho_n)_{n \in \mathbb{N}}$ be a density model and $A$ a state abstraction with abstract state space $\cS_A$. We define a density model $\rho_n^{\text{A}}$ over $\mathcal{S}_A$:
\begin{equation*}
\rho_n^{\text{A}}(\bar{s}, a) = \sum_{s \in G(\bar{s})} \rho_n (s, a) = \frac{\sum_{s \in G(\bar{s})} \hat{N}_{n} (s, a)}{\hat{n}} .
\end{equation*}
Similarly, $\mu^A_n(\bar s, a) := \sum_{s \in G(\bar s)} \mu_n(s, a)$.
We also define a pseudo-count $\hat{N}^{\text{A}}$ and total count $\hat{n}^{\text{A}}$ such that, $\forall \bar{s} \in \mathcal{S}_A, \forall a \in \mathcal{A}$
\begin{equation*}
\rho_n^{\text{A}}(\bar{s}, a) = \frac{\hat{N}^{\text{A}}(\bar{s}, a)}{\hat{n}^{\text{A}}}, \quad \rho_n^{\text{A}'}(\bar{s}, a) = \frac{\hat{N}^{\text{A}}(\bar{s}, a) + 1}{\hat{n}^{\text{A}} + 1} .
\end{equation*}
\end{definition}
We begin with two assumptions on our density model.
\begin{restatable}{assumption}{multiplicative}
\label{ass:multiplicative}
Given an abstraction $A$, there exists constants $a, b, c, d > 0$ such that for all $n \in \mathbb{N}$ and all sequences $s_{1:n}$, $a_{1:n}$ $\forall (\bar{s},a) \in \mathcal{S}_A \times \mathcal{A}$
\begin{align*}
 &(1) \: a \, \mu^{\text{A}}_n (\bs,a) \le \rho^{\text{A}}_n (\bs,a) \le b \, \mu^{\text{A}}_n (\bs,a) \\
 &(2) \:  c \le \frac{\rho^{\text{A}'}_n (\bs,a) - \rho^{\text{A}}_n (\bs,a)}{\mu^{\text{A}'}_n (\bs,a) - \mu^{\text{A}}_n (\bs,a)} \le d .
\end{align*}
\end{restatable}

\begin{restatable}{theorem}{pseudoth}
\label{th:pseudo}
Suppose Assumption \ref{ass:multiplicative} holds. Then the ratio of pseudo-counts $\hNa (\bs,a)$ to empirical counts $N^{\text{A}}_n (\bs,a)$ is bounded and we have
\begin{equation*}
    a^2 c \, N^{\text{A}}_n (\bs,a) \le \hNa (\bs,a) \le b^2 d \, N^{\text{A}}_n (\bs,a) .
\end{equation*}
\end{restatable}
Theorem \ref{th:pseudo} gives a sufficient condition for the pseudo-counts to behave multiplicatively like empirical counts. As already observed by \citet{bellemare2016unifying}, this requires that $\rho^{\text{A}}$ tracks the empirical distribution $\mu^{\text{A}}$, in particular converging at a rate of $\Theta(1/n)$. However, our result allows this rate to vary over time. 

Our result highlights the interplay between the choice of abstraction $A$ and the behaviour of the pseudo-counts. On one hand, applying Assumption \ref{ass:multiplicative} is quite restrictive, requiring that the density model basically match the empirical distribution. By choosing a coarser abstraction we relax this requirement, at the cost of near-optimality.
In Section \ref{sec:implicit_aggregation} we will instantiate the result by viewing the density model as inducing a particular state abstraction.

We now consider the following variant of MBIE-EB:
\begin{equation}
\label{eq:explicit_pseudo}
V (\bs) = \max_{a \in \mathcal{A}} \Bigg[ \hat{\mathcal{R}}_A (\bs, a) + \gamma \expect\limits_{\mathclap{\hat{\mathcal{T}}_A (. | \bs, a)}} V (\bs') + \frac{\beta}{\sqrt{\hNasa}} \Bigg].
\end{equation}
In this variant, the exploration bonus need not match the empirical count. To understand the effect of this change, consider the following two related settings. In the first setting, $\hNasa$ increases slowly and consistently underestimates $\Nasa$. The pseudo-count exploration bonus, which is inversely proportional to $\hNasa$, will therefore remain high for a longer time. In the second setting, $\hNasa$ increases quickly and consistently overestimates $\Nasa$. In turn, the pseudo-count bonus will go to zero much faster than the bonus derived from empirical counts. These two settings correspond to what we call \emph{under-} and \emph{over-exploration}, respectively. We will use Theorem \ref{th:pseudo} to quantify these two effects.

Suppose that $\rho$ satisfies Assumption \ref{ass:multiplicative}, by rearranging the terms, we find that, for any $\alpha > 0$,
\begin{equation*}
  \frac{\alpha / \sqrt{b^2 d}}{\sqrt{N^A_n (s,a)}} \le \frac{\alpha}{\sqrt{\hNa (s,a)}} \le \frac{\alpha / \sqrt{a^2 c}}{\sqrt{N^A_n (s,a)}}
\end{equation*}
Hence the uncertainty over $\hNasa$ carries over to the exploration bonus. Critically, the constant $\beta$ in MBIE-EB is tuned to guarantee that each state is visited at least $m$ times, with probability $1 - \delta$. The following lemma relates a change in $\beta$ with a change in these two quantities.
\begin{restatable}{lemma}{exploration}
\label{lem:exploration}
For $p > 0$, assuming MBIE-EB from Equation \ref{eq:explicit_empi} is run with a bonus $\sqrt{p} \beta  (\Nasa)^{-1/2}$, then
\begin{itemize}
    \item $p < 1$: Proposition \ref{prop:fast_mbie} only holds with probability $1 - \delta /2 - (| \mathcal{S}_A | | \mathcal{A} | m) (\delta / (2 | \mathcal{S}_A  | | \mathcal{A} | m))^{p} $ which is lower than $1 - \delta$. We then say that the agent \emph{under-explores}.
    \item $p > 1$: the sample complexity of MBIE-EB is multiplied by $p$. We then say that the agent \emph{over-explores} by a factor $p$.
\end{itemize}
\end{restatable}
Perhaps unsurprisingly, over-exploration is rather mild, while under-exploration can cause exploration to fail altogether. 
A pseudo-count bonus derived from a density model satisfying the assumption of Theorem \ref{th:pseudo} must under-explore, unless $b = d = 1$ (which implies $\hat N = N$, since $\rho$ is a probability distribution). \\
Lemma \ref{lem:exploration} suggests that we can correct for under-exploration by using a larger constant $\beta$, for $\beta' = \beta \sqrt{b^2 d}$
\begin{equation*}
  \frac{\beta}{\sqrt{\Nasa}} \le \frac{\beta'}{\sqrt{\hNasa}} \le \frac{\sqrt{p} \beta}{\sqrt{\Nasa}} \, \text{with } p = \frac{b^2 d}{a^2 c} .
\end{equation*}

\begin{restatable}{theorem}{variant_with_bonus}
\label{thm:variant_with_bonus}
    Consider a variant \textbf{A'} of MBIE-EB defined with an exploration bonus derived from a density model satisfying the assumption of Theorem \ref{th:pseudo}, and the exploration constant $\beta' = \beta \sqrt{b^2 d}$. Then \textbf{A'}
    \begin{itemize}
        \item does not under-explore, and
        \item over-explores by a factor of at most $\frac{b^2 d}{a^2 c}$.
    \end{itemize}
\end{restatable}
In practice, of course, the value of $\beta$ given by Theorem \ref{th:mbieeb} is usually too conservative and the agent ends up over-exploring. Of note, both \citet{strehl2008analysis} and \citet{bellemare2016unifying} used values ranging from $0.01$ to $0.05$ in their experiments.

\section{Implicitly approximate exploration}\label{sec:implicit_aggregation}

In previous sections we studied an algorithm which is aware of, and takes into account, the state abstraction. In practice, however, bonus-based methods have been combined to a number of function approximation schemes; as noted by \citet{bellemare2016unifying}, the degree of compatibility between the value function and the exploration bonus is sure to impact performance. We now combine the ideas of the two previous sections and study how the particular density model used to generate pseudo-counts induces an implicit approximation.

When does Assumption \ref{ass:multiplicative} hold? In general we cannot expect it to be valid for any given abstraction, in particular it is unrealistic to hope that it will be verified in the ground state space. On the other hand, it is natural to assume that there exists an abstraction defined by the density model which satisfies the assumption with reasonable constants. In this section we will see that a density model defines an \emph{induced abstraction}. In turn, we will quantify how this abstraction provides us with a handle into Assumption \ref{ass:multiplicative}.

\subsection{Induced abstraction}
From a density model, we define a state abstraction function as follows.
\begin{definition}
For $\epsilon \in [0, 1[$, the \emph{induced abstraction} $\phi_{\rho,\epsilon}$ is such that
\begin{equation}
\begin{gathered}
\forall s, s' \in \mathcal{S}_G \quad \phi_{\rho,\epsilon} (s) = \phi_{\rho,\epsilon} (s') \rightarrow \forall n \in \mathbb{N}, \, \forall a \in \mathcal{A}, \\
 \quad (1 - \epsilon) \rho_n (s', a) \leq \rho_n (s, a) \leq (1 + \epsilon) \rho_n (s',a), \\
1 - \epsilon \le \frac{\rho_n' (s,a) - \rho_n (s,a)}{\rho_n '(s',a) - \rho_n (s',a)} \le 1 + \epsilon .
\nonumber
\end{gathered}
\end{equation}
\end{definition}
In words, two ground states $s, s'$ are aggregated if the density model always assigns a close likelihood to both for each action. For example, this is the case when the visit counts of nearby states in a grid world are aggregated together; we will study such a model shortly. The definition of this abstraction is independent of the sequence of states the model was trained on and only depends on the model. From this definition co-aggregated states have similar pseudo-count, from Theorem \ref{th:pseudo}, for two ground states $s, s'$
\begin{equation*}
    (1 - \epsilon)^3 \hN(s', a) \leq \hN(s, a) \leq (1 + \epsilon)^3 \hN(s', a).
\end{equation*}

Suppose that the induced pseudo-count $\hat N^{\text{A}}$ satisfies Assumption \ref{ass:multiplicative}. One may expect that this is sufficient to obtain similar guarantees to those of Theorem \ref{thm:variant_with_bonus}, by relating the ground pseudo-count $\hat N$ (computed from $\rho$) to the abstract pseudo-count $\hat N^{\text A}$ (which we could compute from $\rho_n^{\text A}$). In particular, for a small $\epsilon$, we may expect the following relationship
\begin{equation*}
    \hNsa = \frac{\hNasa }{| G(s) |},
\end{equation*}
whereby an abstract state's pseudo-count is divided uniformly between the pseudo-counts of the states of the abstraction. Surprisingly, this is not the case, and in fact as we will show $\hat N$ is greater than its corresponding $\hat N^{\text{A}}$. The following makes this precise:

\begin{restatable}{lemma}{approxineq}
Let $\rho$ be a density model and $\ra$, $\hN$, $\hNa$, and $\hna$ as before. Then for $a \in \mathcal{A}$ and $s \in \mathcal{S}_G$
\begin{equation*}
\;  \hNasa \cdot f(\bar{s}, a, \epsilon) \leq \hNsa \leq \hNasa \cdot g(\bar{s}, a, \epsilon).
\end{equation*}
with $\alpha_\epsilon = \frac{1 - \epsilon}{1 + \epsilon}$, $f$ and $g$ are given by:
\small{
\begin{align*}
f(\bar{s}, a, \epsilon) &= \frac{| G(s) | (\hna + 1) - (1 + \epsilon)^3 (\hNasa + 1) }{| G(s) | (\tfrac{1}{\alpha_\epsilon^3} \hna -  \hNasa + (\tfrac{1}{\alpha_\epsilon^3} - 1) \hna \hNasa)}  \\
g(\bar{s}, a, \epsilon) &= \frac{| G(s) | (\hna + 1) - (1 - \epsilon)^3 (\hNasa + 1) }{| G(s) | (\alpha_\epsilon^3 \hna -  \hNasa - (1 - \alpha_\epsilon^3) \hna \hNasa)}
\end{align*}
}
\end{restatable}

\begin{restatable}{corollary}{abstoground}
\label{lem:abstoground}
For an exact abstraction ($\epsilon = 0$)
\begin{equation*}
\hNsa = \hNasa \Bigg( 1 + \frac{(| G(s) | -1)(\hNasa + 1) }{| G(s) | (\hna -  \hNasa)} \Bigg).
\end{equation*}
\end{restatable}

Two remarks are in order. First, for any kind of aggregation, $| G(s) | > 1$ implies $\hat{N}(s, a) > \hat{N}^\text{A} (\bar{s}, a)$. Second, $\hat N(s, a) - \hat{N}^{\text{A}} (\bar{s}, a) \to + \infty$ when $\hat{N}^{\text{A}} (\bar{s}, a) \to \hna$, that is, as the density concentrates within a single aggregation, then the pseudo-counts for individual states grow unboundedly.
Our result highlights an intriguing property of pseudo-counts: when the density model generalizes (in our case, by assigning the same probability to aggregated states) then the pseudo-counts of individual states increase faster than under the true, empirical density.

One particularly striking instance of this effect occurs when $\rho$ is itself defined from an abstraction $\phi$. That is, consider the density model $\rho$ which assigns a uniform probability to all states within an aggregation:
\begin{equation}
\label{eq:density_abs}
    \rho_n (s, a) = \frac{N^{\text A}_n (\phi(s), a)}{|G(s)| n} .
\end{equation}
Lemma \ref{lem:abstoground} applies and we deduce that the pseudo-count associated with $s$ is greater than the visit count for its aggregation: $\hNsa > N^{\text A}_n (s, a)$.
From Lemma \ref{lem:abstoground} we conclude that, unless the induced abstraction is trivial, we cannot prevent under-exploration when using a pseudo-count based bonus. One way to derive meaningful guarantees is to bound the lemma's multiplicative constant, by requiring that no abstraction be visited too often.
\begin{restatable}{proposition}{pseudobound}
\label{prop:pseudobound}
Consider a state $s \in \mathcal{S}_G$. If there exists $k > 1$ such that $0 \le \hNsa \le \tfrac{\hna}{k}$ then for an exact abstraction:
\begin{equation*}
\hNasa \leq \hNsa \le \hNasa \Big( 1 + \frac{2}{k-1} \Big).
\end{equation*}
\end{restatable}
In particular this result justifies how pseudo-counts generalize to unseen states, while a pair $(s, a)$ may have not been observed, the pseudo-count $\hNasa$ will increases as long as other pairs in the same aggregation are being visited.
 
One way to guarantee the existence of a uniform constant $k$ in Proposition \ref{prop:pseudobound} is to inject random noise in the behaviour of the agent, for example by acting $\epsilon$-greedily with respect to the MBIE-EB $Q$-values. In this case, a bound on $k$ can be derived by considering the rate of convergence to the stationary distribution of the induced Markov chain \citep[see e.g.][]{fill91eigenvalue}.

\begin{figure*}[!tbp]
\centering
\begin{subfigure}{.4\textwidth}
    \centering
  \begin{tikzpicture}[auto,node distance=10mm,>=latex,font=\small]

    \tikzstyle{round}=[thick,draw=black,circle]
    
    \node[round] (s0) {$s_0$};
    \node[round,below=4mm of s0] (s1) {$s_1$};
    \node[round,below=10mm of s1] (s2) {$s_{t}$};
    \node[round,below right=0mm and 15mm of s1] (s5) {$T_1$};
    \node[round,below left=0mm and 15mm of s1] (s6) {$T_0$};
    
    \draw[black,dashed](0,-1.6)--(0,-2.3);
    \node[below=0.15cm of s5] {$R = L$};
    \node[below=0.15cm of s6] {$R = \epsilon$};
                 
    \draw[->, red] (s0) edge node{$1$} (s6);
    \draw[->, red] (s1) edge node{} (s6);
    \draw[->, red] (s2) edge node{} (s6);
    
    \draw[->,blue] (s0) edge node{} (s5);
    \draw[->, blue] (s1) edge node{} (s5);
    \draw[->, blue] (s2) edge node{$p$} (s5);
    \draw[<-,blue] (s0) edge[loop right] node{$1 - p$} (s0);
    \draw[blue] (s1) edge[loop right] node{} (s1);
    \draw[blue] (s2) edge[loop right] node{} (s2);
  \end{tikzpicture}
  \caption{Challenging MDP for exploration. Transitions for action \emph{left} are in red and in blue for action \emph{right}.}
  \label{fig:graph_over}
\end{subfigure}%
\hfill
\begin{subfigure}{.5\textwidth}
\centering
  \includegraphics[width=.65\linewidth]{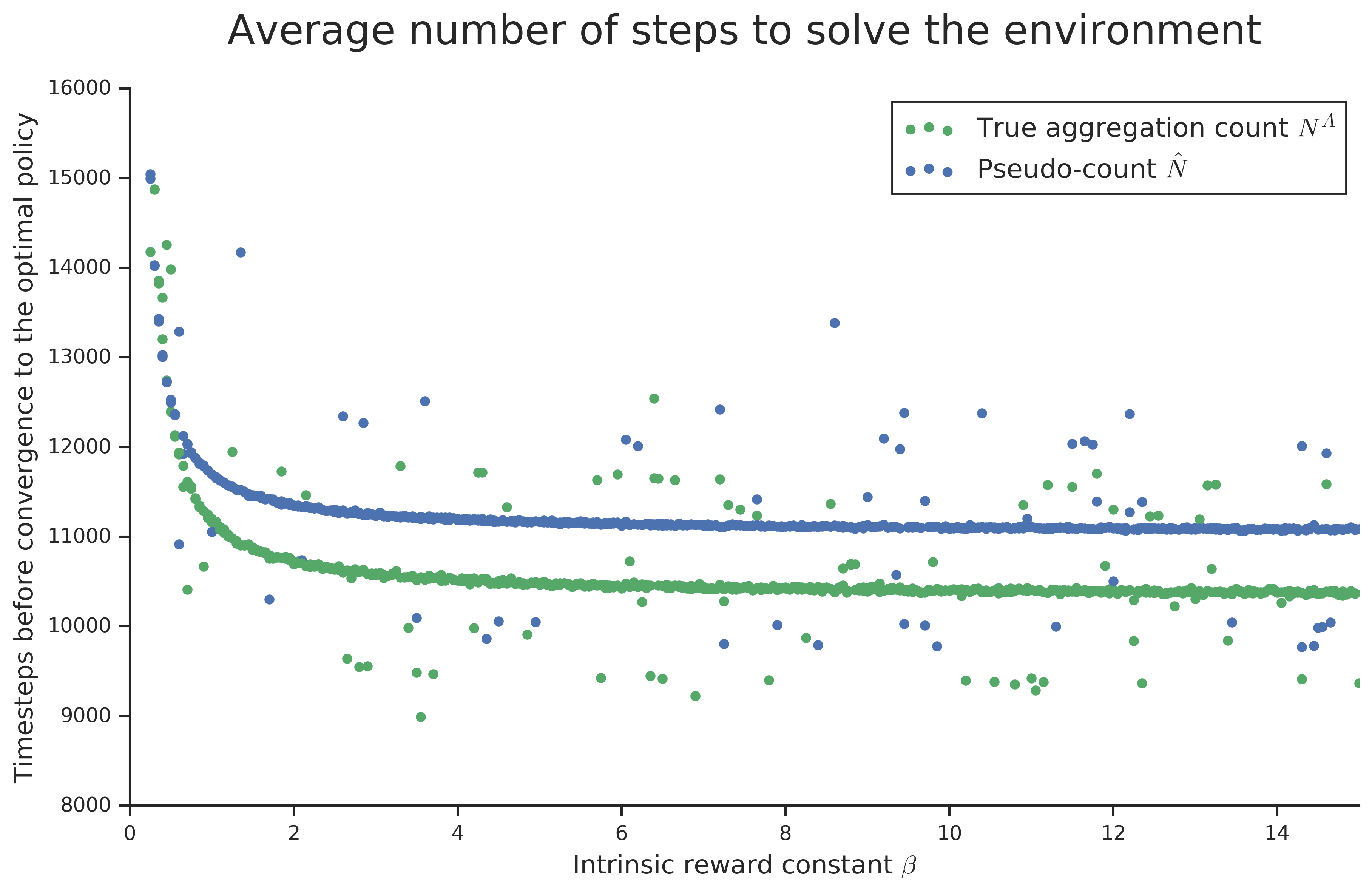}
  \caption{Average time to convergence to the optimal policy as a function of $\beta$.}
  \label{fig:result_over}
\end{subfigure}%
\caption{Quantifying the impact of pseudo-counts over-estimation on exploration}
\label{fig:overestimation}
\end{figure*}
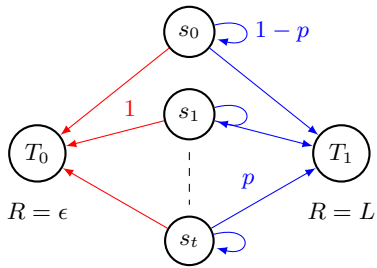
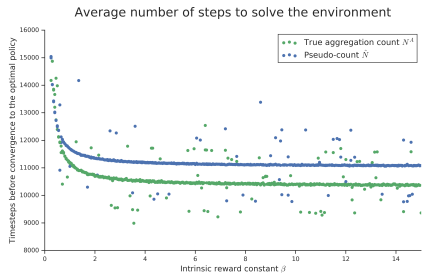

\subsection{Over-estimation impact on exploration}

We provide now an example of MDP (see Figure \ref{fig:graph_over}) where the overestimation described previously can hurt exploration. \\
In this example, the initial state distribution is uniform over states $s_0, ..., s_t$. Each episode lasts for a single timestep. The agent can either choose the action \emph{left}, transition to $T_0$ collecting a small reward $\epsilon$ or choose the action \emph{right} which leads to state $T_1$ with probability $p$  collecting a reward $L \gg \epsilon$, otherwise, the agent remains at the same state. In this setting it seems natural to aggregate states $s_0, ..., s_t$ as they share similar properties. \\
We apply MBIE-EB on this environment and compare pseudo-counts derived from a density model similar to Equation \ref{eq:density_abs} with the empirical count of the aggregation. From Corollary \ref{lem:abstoground} we know that $\hat N_n (s_i, a) > N^{\text{A}}_n (\bar{s}_i, a)$ for any action $a$ and state $s_i$, furthermore at the beginning of training, as the agent explores and alternate between the two actions at similar frequency, the overestimation grows linearly. When $p$ is small this can induce the agent to under-explore and choose the sub-optimal action \emph{left}. We run our example with $L = 100, \, \epsilon = 0.001, \, p=1/10000$, action \emph{left} value is $0.001$ whereas it is $0.01$ for action \emph{right}. Figure \ref{fig:result_over} shows the time to converge to the optimal policy over 20 seeds for different values of MBIE-EB constant $\beta$. While this example is pathological, it shows the impact pseudo-count over-estimation can have on exploration, in the next section we provide a way around this issue. \\

\subsection{Correcting for counts over-estimation}
The over-estimation issue detailed in Corollary \ref{lem:abstoground} is a consequence of pseudo-counts postulating in their defintion that the count of a single state should increase after updating the density model. In practice when a state is visited the count of every other state in the same aggregation should increase too. It is possible to derive a new pseudo-count bonus verifying this property as we shall see now
\begin{restatable}{theorem}{improvedpseudo}
\label{improved_pseudo}
Let $\tN$ be the pseudo-count defined such that for any state-action pair $(s, a)$
\begin{equation*}
\label{eq:real_sys}
    \rho_n (s,a) = \frac{\tNsa}{\tilde{n}}, \quad  \rho_n '(s, a) = \frac{\tNsa + 1}{\tilde{n}+ |G(s)|}.
\end{equation*}
with $\tilde{n}$ the pseudo-count total. Then $\tN$ can be computed as follow
\begin{equation*}
\tNsa = \frac{2 \rsa \tau'_n (s,a)}{\rppsa \tau_n (s,a) - \rsa \tau'_n (s,a)}.
\end{equation*}
where $\rppsa \coloneqq \rho'(s, a; s_{1:n}s, a_{1:n}a)$ and $\tau_n (s,a) = \rpsa - \rsa$. \\
For an exact induced abstraction, $\tilde N$ does not suffer from the over-estimation previously mentioned and we have 
$$\tNsa = \tilde{N}_n^\text{A} (\bar{s}, a) = \hNasa.$$
for any state action pair $(s, a)$.
\end{restatable}

Theorem \ref{improved_pseudo} shows that it is possible to mitigate pseudo-counts over-estimation at the cost of more compute as the density model needs to be updated twice at each timestep. For the density model defined from an abstraction in Equation \eqref{eq:density_abs} the pseudo count $\tN$ will this time exactly match the count of the abstraction. It should also be noted that we have $\tN \leq \hN$, so any reward bonus derived from $\tN$ will be higher than if it was derived from $\hN$ instead which may be beneficial in the function approximation where the intrinsic reward would provide more signal.

\begin{figure*}
\centering
\begin{subfigure}{.5\textwidth}
    \centering
  \includegraphics[width=.35\linewidth]{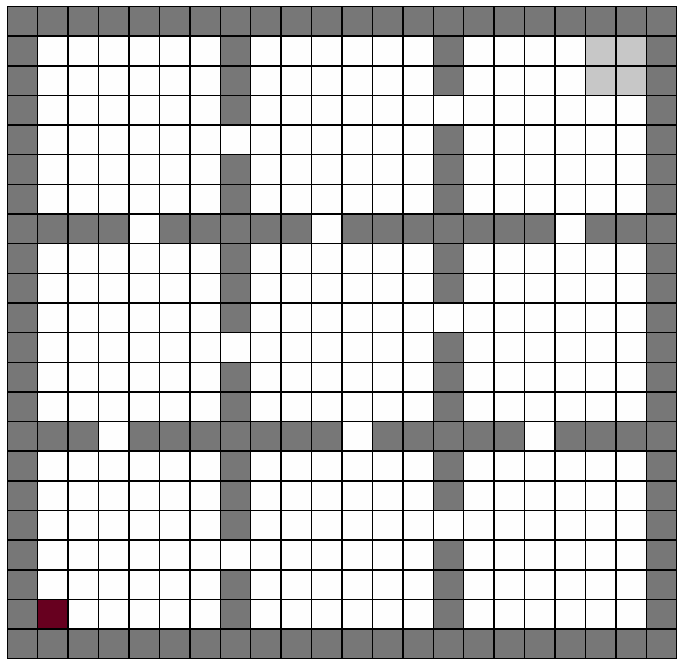}
  \caption{9-room domain}
  \label{fig:grid}
\end{subfigure}%
\hfill
\begin{subfigure}{.5\textwidth}
\centering
  \includegraphics[width=.35\linewidth]{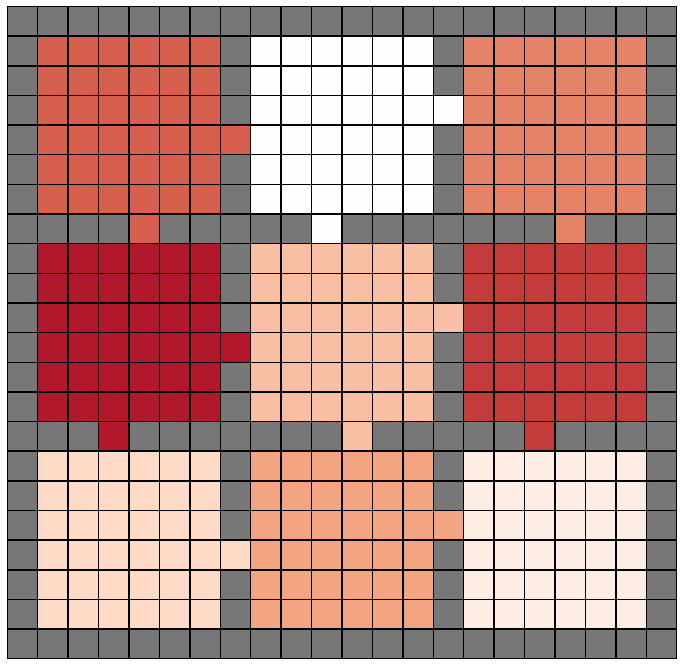}
  \caption{State abstraction defined by the density model}
  \label{fig:agg}
\end{subfigure}%
\caption{Domain used for evaluation}
\label{fig:domain}
\end{figure*}

\begin{figure*}
\centering
\captionsetup{justification=centering}
\begin{subfigure}{.25\textwidth}
    \centering
  \includegraphics[width=.9\linewidth]{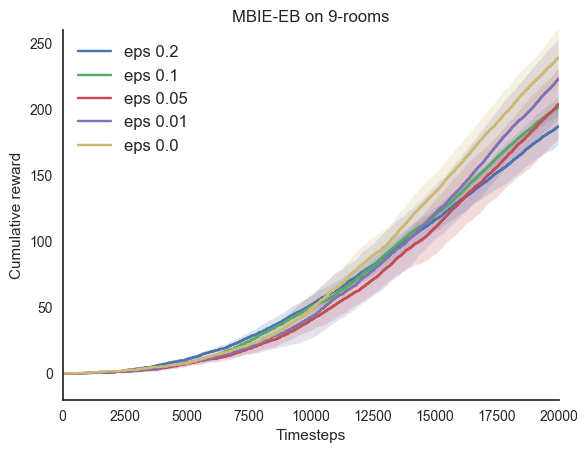}
  \caption{MBIE-EB varying $\epsilon$}
  \label{fig:sub1}
\end{subfigure}%
\begin{subfigure}{.25\textwidth}
\centering
  \includegraphics[width=.9\linewidth]{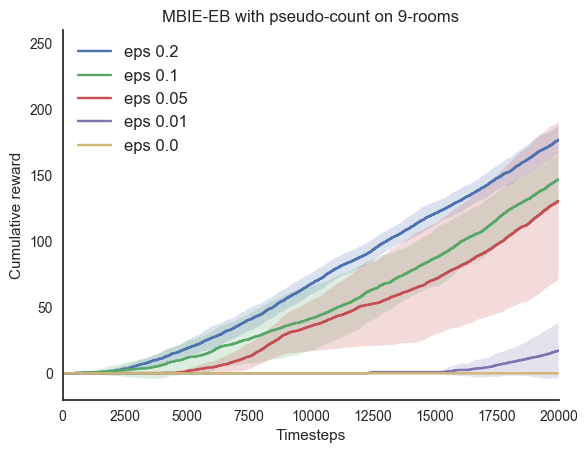}
  \caption{MBIE-EB-PC varying $\epsilon$}
  \label{fig:sub2}
\end{subfigure}%
\begin{subfigure}{.25\textwidth}
\centering
  \includegraphics[width=.9\linewidth]{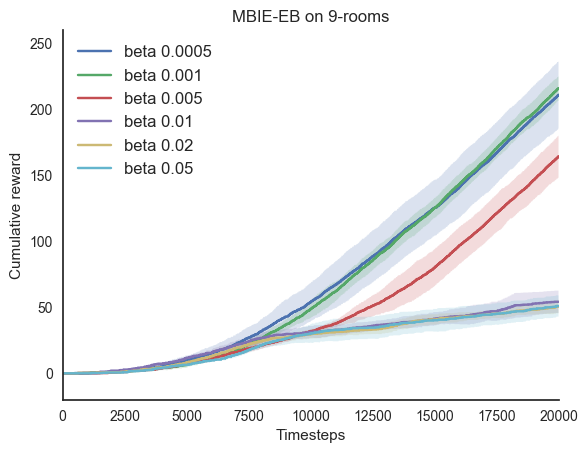}
  \caption{MBIE-EB varying $\beta$}
  \label{fig:sub3}
\end{subfigure}%
\begin{subfigure}{.25\textwidth}
\centering
  \includegraphics[width=.9\linewidth]{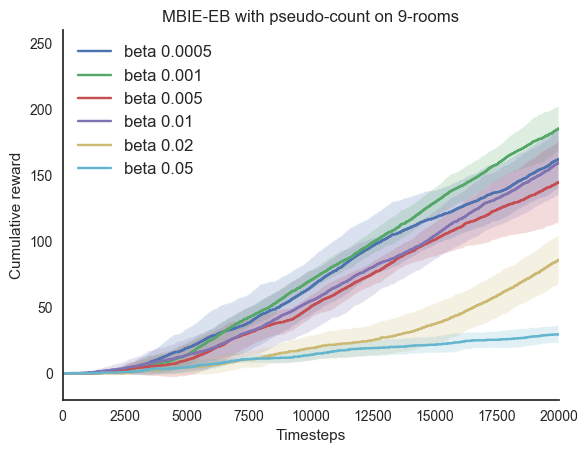}
  \caption{MBIE-EB-PC varying $\beta$}
  \label{fig:sub4}
\end{subfigure}%
\caption{Rewards accumulated by the agent on the 9-room domain. Figures \ref{fig:sub1} and \ref{fig:sub2} use a fixed value of $\beta = 1e-4$, where Figures \ref{fig:sub1} and \ref{fig:sub2} vary $\beta$ and use a fixed $\epsilon = 0.1$.}
\label{fig:exp}
\end{figure*}

\subsection{Empirical evaluation}
Combining Theorem \ref{th:pseudo} and Lemma 2 (or applying Theorem \ref{improved_pseudo}) we can bound the ratio of pseudo counts to empirical counts $\Nasa$ for a given abstraction verifying Assumption 1. Nevertheless the impact of a bonus derived from an abstraction to explore in the ground state space has not been quantified. This was referred to by \citet{bellemare2016unifying} as the lack of compatibility between the exploration bonus and the value function. While we were not able to derive theoretical results regarding this particular case, we provide an empirical study on a grid world.

We use a 9-room domain (see Figure ~\ref{fig:grid}) where the agent starts from the bottom left and needs to reach one of four top right states to receive a positive reward of 1. The agent has access to four actions: \emph{up}, \emph{down}, \emph{left}, \emph{right}. Transitions are deterministic; moving into walls leaves the agent in the same position. The environment runs until the agent reaches the goal, at which point the agent is rewarded and the episode  starts over from the initial position.

We compare MBIE-EB using the empirical count from Equation \eqref{eq:mbieeb} with the variant of MBIE-EB using pseudo-counts bonuses - MBIE-EB-PC - from Equation \eqref{eq:explicit_pseudo} 
\footnote{We did not notice any significant difference between $\hat N$ and $\tilde{N}$ and used $\hat N$ for all experiments}.
Pseudo-counts are derived from a density model (Equation ~\ref{eq:density_abs}) which assigns a uniform probability to states within the same room as shown in Figure~\ref{fig:agg}.
We also investigate the impact of an $\epsilon$-greedy policy as proposed in the previous subsection.
Figure \ref{fig:exp} depicts the cumulative rewards received by both our agents for different values of $\beta$ and $\epsilon$.  Each experiment is averaged over 5 seeds, shaded error represents variance. It demonstrates that:
\begin{itemize}
    \item MBIE-EB fulfills the task relatively well in most instances while the lack of compatibility between the value function and the pseudo count exploration bonus can impact performance to the point where MBIE-EB-PC fails completely (Figure \ref{fig:sub2}).
    \item While MBIE-EB is not much affected by the $\epsilon$-greedy policy, the $\epsilon$ parameter is critical for MBIEB-EB-PC. While the pseudo-count bonus provides a signal to explore across room., a high value of $\epsilon$ is necessary for the agent to maneuver within individual rooms. In order to avoid under-exploration, higher values of $\epsilon$ work best.
    \item By not assigning a count to every state action pair, MBIE-EB-PC can act greedily with respect to environment and achieves a higher cumulative reward in the first 10,000 timesteps than MBIE-EB.
    \item MBIE-EB-PC is more robust to a wider range of values of $\beta$, suggesting that exploration in the ground MDP is more subject to over-exploration.
\end{itemize}

\section{Related Work}
Performance bounds for efficient learning of MDPs have been thoroughly studied. In the context of PAC-MDP algorithms, model-based approaches such as \textsc{Rmax} \citep{brafman2002r}, MBIE and MBIE-EB \citep{strehl2008analysis} or $\text{E}^3$ \citep{kearns2002near} build an empirical model of a set of the environment state-actions pairs using the agent's past experience. \citet{strehl2006pac} also investigated the model-free case with delayed Q-learning and showed that they could lower the sample complexity dependence on state space dimension. Bayesian Exploration Bonus proposed by \citet{kolter2009near} is not PAC-MDP but offers the guarantee to act optimally with respect to the agent's prior except for a polynomial number of timesteps.

In the average reward case, UCRL \citep{jaksch2010near} was shown to obtain a low regret on MDPs with finite diameter. Many extensions exploit the structure of the MDP to improve further the regret bound \citep{ortner2013adaptive, osband2014near, hutter2014extreme, fruit2018efficient, NIPS2018_8103}. Similarly \citet{kearns1999efficient} presented a variant of $E^3$ which is also PAC-MDP. Temporal abstraction in the form of extended actions  \citep{sutton1999between} has been recently studied for exploration. \citet{brunskill2014pac} proposed a variant of \textsc{Rmax} for SMDPs and \citet{fruit2017exploration} extended UCRL to MDPs where a set of options is available, both have shown promising results when a good set of options is available.

Finding abstractions in order to handle large state spaces remains a long standing goal in reinforcement learning, a lot of work in the literature has been focused on finding metrics to quantify state similarity \citep{bean1987aggregation,andre2002state}. \citet{li2006towards} provided an unifying view on exact abstractions that preserve the optimality. Metrics related to the model similarity metric include bisimulation \citep{ferns2004metrics,ferns2012methods}, bounded parameters MDPs \citep{givan2000bounded}, $\epsilon$-similarity \citep{even2003approximate,ortner2007pseudometrics}.

\subsection*{Conclusion}
In this work we build on previous results related to state abstraction and exploration. We highlighted how they can help to understand better the success of exploration using pseudo-counts in the non-tabular case. As it turns out, with  finite time, optimal exploration might be too hard to obtain and we have to settle for approximate solution that trade off speed convergence and guarantee w.r.t to the policy learned.

It is unlikely that practical exploration will enjoy near-optimality guarantees as powerful as those given by theoretical methods. In most environments, there are simply too many places to get lost.
Alternative schemes -- such as the value-based exploration idea proposed by \citet{leike16exploration} -- may help but only so much.
In our work, we showed that abstractions allow us to impose a certain prior on the shape that exploration needs to take.

We also found that pseudo-count based methods, like other abstraction-based schemes, can fail dramatically when they are incompatible with the environment. While this is expected given their practical trade-off, we believe our work moves us towards a better understanding of bonus-based methods in practice. 
An interesting question is whether adaptive schemes can be designed that would enjoy both the speed of exploration of coarse abstractions with the near-optimality guarantees of fine ones.

\subsection*{Acknowledgements}
We would like to thank Mohammad Azar, Sai Krishna, Tristan Deleu, Alexandre Piché, Carles Gelada and Michael Noukhovitch for careful reading and insightful comments on an earlier version of the paper. This work was funded by FRQNT through the CHIST-ERA IGLU project.

\bibliographystyle{plainnat}
\bibliography{bibli}

\newpage
\begin{appendices}
\onecolumn


\section{Proofs}
\begin{lemma} \label{lemma:bound}
For a model similarity abstraction we have the following inequality:
\begin{gather*}
\forall s \in \Sg, \forall a \in \mathcal{A}, \quad | Q_G (s, a) - Q_A (\phi_{\eta}(s), a) | \le \frac{\eta + \gamma (| \Sa | - 1) \eta}{(1 - \gamma)^2}.
\end{gather*}
\end{lemma}

\begin{proof}
Note that we have the following inequalities
\begin{gather*}
\forall s \in \Sg, a \in \mathcal{A} \quad  \vert \Rg (s, a) - \Ra (\bs, a) \vert \leq \eta, \\
\forall \bar{s}, \bar{s}' \in \Sa, a \in \mathcal{A}, g \in G(\bar{s}) \quad \vert \Ta (\bs, a, \bs') - \sum_{g' \in G(\bs')} \Tg (g,a,g') \vert \leq \eta.
\end{gather*}
Then
\begin{align*}
 | &Q_G (s, a) - Q_A (\bs, a) | \leq \\
 &| \Rg (s, a) - \Ra (\bs, a)  | + \gamma \bigg\lvert \sum_{\bs' \in \Sa} \big[ \sum_{s ' \in G(\bs')} \Tg (s, a, s') \max_{a'} Q_G(s', a')  \big] - \Ta (\bs, a, \bs') \max_{a''} Q_A (\bs', a'')  \bigg\rvert \\
&\leq \eta + \gamma \bigg\lvert \sum_{\bs' \in \Sa} \big[ \sum_{s ' \in G(\bs')} \Tg (s, a, s') \big( \max_a Q_G(s', a') - \max_{a''} Q_A (\bs', a'') + \max_{a''} Q_A (\bs', a'') \big)  \big] - \Ta (\bs, a, \bs') \max_{a''} Q_A (\bs', a'')  \bigg\rvert \\
&\leq \eta + \gamma \bigg\lvert \sum_{\bs' \in \Sa} \max_{a''} Q_A (\bs', a'') \Big[ (\sum_{s ' \in G(\bs)}  \Tg (s, a, s') \big) -  \Ta (\bs, a, \bs) \Big] \\
&+ \sum_{s ' \in G(\bs')}  \Tg (s, a, s') (\max_{a'} Q_G(s', a') - \max_{a''} Q_A (\bs', a''))  \big] \bigg\rvert \\
&\leq \eta + \frac{\gamma \eta  | \Sa  |}{1 - \gamma} + \gamma  \sum_{\bs' \in \Sa} \sum_{s' \in G(\bs)}  \Tg (s, a, s') \max_{a'}  | Q_G(s', a') - Q_A (\bs', a')  | \\
&\leq \eta + \frac{\gamma \eta  | \Sa  |}{1 - \gamma} + \gamma \max_{\bs', s' \in G(\bs'), a}  | Q_G(s', a) - Q_A (\bs', a)  |
\end{align*}
\end{proof}

\begin{remark}
The previous Lemma can be used to show a model similarity abstraction has sub-optimality bounded in $\eta$ and improves the bound of \citet{abel2016near}, which has a $1/(1 - \gamma)^3$ dependency, due to an issue in the original proof. To the best of our knowledge, ours is the first complete proof of this result.
\end{remark}

\begin{lemma}
A model similarity abstraction (Def. \ref{defn:model_similiarity}) has sub-optimality bounded in $\eta$
\begin{equation}
\label{eq:transfer}
\forall s \in \Sg, \, V_G (s) - V_G^{\pi_{GA}} (s) \le \frac{2\eta + 2\gamma (| Sa | - 1) \eta}{(1 - \gamma)^2}.
\end{equation}
\end{lemma}

\begin{proof}
Using similar arguments than in Lemma \ref{lemma:bound} we can show that:
\begin{equation*}
\begin{aligned}
| V^{\pi_{GA}}_G (s) -  V_A (\phi (s)) | &=  | Q_G^{\pi_{GA}} (s, \pi_{A}^* (\bar{s})) - Q_A (\bar{s}, \pi_{A}^* (\bar{s})) | \\
&\le \frac{\eta + \gamma (| \mathcal{S}_A | - 1) \eta}{(1 - \gamma)^2}
\end{aligned}
\end{equation*}
Then using Lemma \ref{lemma:bound} again, we have:
\begin{equation*}
\begin{aligned}
|  V_G (s) -  V_A (\bar{s}) | &= | \max_a Q_G (s, a) - \max_{a'} Q_A (\bar{s}, a') | \\
&\le \max_a | Q_G (s, a) - Q_A (\bar{s}, a) | \\
&\le \frac{\eta + \gamma (| \mathcal{S}_A | - 1) \eta}{(1 - \gamma)^2} \\
\end{aligned}
\end{equation*}
And we can conclude:
\begin{equation*}
\begin{aligned}
|  V_G (s) -  V^{\pi_{GA}}_G (s) | &\le | V_G (s) - V_A (\bar{s}) | + | V_A (\bar{s}) -  V^{\pi_{GA}}_G (s) | \\
&\le \frac{2\eta + 2\gamma (| \mathcal{S}_A | - 1) \eta}{(1 - \gamma)^2}
\end{aligned}
\end{equation*}

\end{proof}

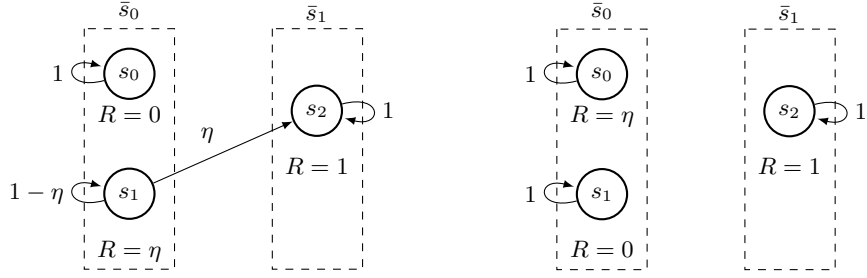
\begin{figure}[!tbp] 
  \centering
  \begin{tikzpicture}[auto,node distance=8mm,>=latex,font=\small]

    \tikzstyle{round}=[thick,draw=black,circle]
    
    \node[round] (s0) {$s_0$};
    \node[round,below right=0mm and 20mm of s0] (s2) {$s_2$};
    \node[round,below=9mm of s0] (s1) {$s_1$};
    
    \path (s1) edge[loop left] node{$1 - \eta$} (s1);
    \path (s0) edge[loop left] node{$1$} (s0);
    \path (s2) edge[loop right] node{$1$} (s2);
    \path [->] (s1) edge node{$\eta$} (s2);
    \node[thick,above=0.6cm] {$\bar{s}_0$};
    \node[thick,above=0.7cm of s2] {$\bar{s}_1$};
    \node[below=0.3cm] {$R = 0$};
    \node[below=0.15cm of s2] {$R = 1$};
    \node[below=0.15cm of s1] {$R = \eta$};
    \draw[dashed] (-0.6,0.6) rectangle (0.6,-2.6);
    \draw[dashed] (1.9,0.6) rectangle (3.1,-2.6);
  \end{tikzpicture}
  \qquad \qquad
  \begin{tikzpicture}[auto,node distance=8mm,>=latex,font=\small]

    \tikzstyle{round}=[thick,draw=black,circle]
    
    \node[round] (s0) {$s_0$};
    \node[round,below right=0mm and 20mm of s0] (s2) {$s_2$};
    \node[round,below=9mm of s0] (s1) {$s_1$};
    
    \path (s1) edge[loop left] node{$1$} (s1);
    \path (s0) edge[loop left] node{$1$} (s0);
    \path (s2) edge[loop right] node{$1$} (s2);
    \node[thick,above=0.6cm] {$\bar{s}_0$};
    \node[thick,above=0.7cm of s2] {$\bar{s}_1$};
    \node[below=0.3cm] {$R = \eta$};
    \node[below=0.15cm of s2] {$R = 1$};
    \node[below=0.15cm of s1] {$R = 0$};
    \draw[dashed] (-0.6,0.6) rectangle (0.6,-2.6);
    \draw[dashed] (1.9,0.6) rectangle (3.1,-2.6);

\end{tikzpicture}
  \caption{Rewards and transitions for $a_1$ (left) and $a_2$ (right) in the ground MDP}
  \label{fig:ground_mdp}
\end{figure}

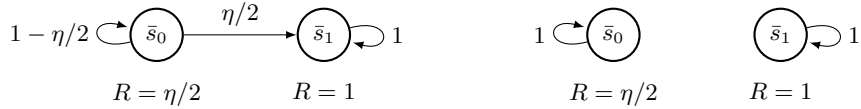
\begin{figure}[!tbp] 
  \centering
  \begin{tikzpicture}[auto,node distance=8mm,>=latex,font=\small]

    \tikzstyle{round}=[thick,draw=black,circle]
    
    \node[round] (s0) {$\bar{s}_0$};
    \node[round,right=15mm of s0] (s1) {$\bar{s}_1$};
    
    \path (s0) edge[loop left] node{$1 - \eta /2$} (s0);
    \path [->] (s0) edge node{$\eta / 2$} (s1);
    \path (s1) edge[loop right] node{$1$} (s1);
    \node[below=0.5cm] {$R = \eta / 2$};
    \node[below=0.15cm of s1] {$R = 1$};
    \end{tikzpicture}
  \qquad \qquad
  \begin{tikzpicture}[auto,node distance=8mm,>=latex,font=\small]

    \tikzstyle{round}=[thick,draw=black,circle]
    
    \node[round] (s0) {$\bar{s}_0$};
    \node[round,right=15mm of s0] (s1) {$\bar{s}_1$};
    
    \path (s0) edge[loop left] node{$1$} (s0);
    \path (s1) edge[loop right] node{$1$} (s1);
    \node[below=0.5cm] {$R = \eta / 2$};
    \node[below=0.15cm of s1] {$R = 1$};
    \end{tikzpicture}
  \caption{Rewards and transitions for $a_1$ (left) and $a_2$ (right) in the abstract MDP}
  \label{fig:abstract_mdp}
\end{figure}

\prop*
\begin{proof}
Consider a three states MDP with two actions $a_1$ and $a_2$ (Figure \ref{fig:ground_mdp}). When states $\{s_0, s_1\}$ and $\{s_2\}$ are aggregated in the abstract states $\bar{s}_0$ and $\bar{s}_1$ this MDP defines a model similarity abstraction of parameter $\eta$ (Figure \ref{fig:abstract_mdp}). In $\bar{s}_0$ we can either choose the policy $\pi_1$ such that $\pi_1 (\bar{s}_0) = a_1$ or  $\pi_2$ such that $\pi_2 (\bar{s}_0) = a_2$. Using the Belleman equation we can compute the value of $\bar{s}_0$ under each policy which yields:
\begin{align*}
V^{\pi_1} (\bar{s}_0) =  \frac{\eta}{2(1 - \gamma)(1 -\gamma + \gamma \eta / 2)}, && V^{\pi_2} (\bar{s}_0) =  \frac{\eta}{2(1 - \gamma)}
\end{align*}
Then $V^{\pi_1}(\bar{s}_0) > V^{\pi_2} (\bar{s}_0)$ means that $a_1$ is the optimal action in $\bar{s}_0$. On the other hand in the ground MDP, $a_2$ is the optimal action in $s_0$ as its value is $\eta / (1 - \gamma)$ and $a_1$ value is zero and choosing $a_1$ is $\epsilon = \eta / (1 - \gamma)$ non-optimal
\end{proof}

\fastmbie*
\begin{proof}
We appeal twice to the triangle inequality to relate the optimal value function in $M_G$ successively to the optimal value function in $M_A$ and to the policy $\tilde{\pi}_A$ produced by MBIE applied to $M_A$:
\begin{equation*}
    | V_G (s) -  V^{\tilde{\pi}_{GA}}_G (s) | \le | V_G (s) - V_A (\bar{s})  | + | V_A (\bar{s}) -  V^{\tilde{\pi}_{A}}_A (\bar{s})| +  | V_A^{\tilde{\pi}_{A}} (\bar{s}) - V^{\tilde{\pi}_{GA}}_G (s)| .
\end{equation*}
We know that the first and third terms in the inequality above are no greater than $g(\eta)$. By our choice of $\epsilon$, the middle term is also guaranteed to be of the same order.
\end{proof}

\pseudoth*
\begin{proof}
From the definition of $\hN (s)$ and $N(s)$:
\begin{equation*}
    \begin{aligned}
    \frac{\hN (s)}{N (s)} &= \frac{\rho_n (s) (1 - \rho_n^{'} (s))}{N_n (s) (\rho_n^{'} (s) - \rho_n (s))} \\
    &= \frac{\rho_n (s) (1 - \rho_n^{'} (s))}{n \mu_n (s) (\rho_n^{'} (s) - \rho_n (s))} \\
    &= \frac{\rho_n (s) (\mu_n^{'} - \mu_n (s))}{\mu_n (s) (\rho_n^{'} (s) - \rho_n (s))} \frac{(1 - \rho_n^{'} (s))}{n (\mu_n^{'} (s) - \mu_n (s))} \\
    &= \frac{\rho_n (s)}{\mu_n (s)} \frac{\mu_n^{'} - \mu_n (s)}{\rho_n^{'} (s) - \rho_n (s)} \frac{1 - \rho_n^{'} (s)}{1 - \mu_n^{'} (s)}
    \end{aligned}
\end{equation*}
Using $n (\mu_n^{'} (s) - \mu_n (s)) = 1 - \mu_n^{'} (s)$ (Lemma 1 from \citet{bellemare2016unifying}), the result follows from:
\begin{equation*}
    \frac{1 - \rho_n^{'} (s)}{1 - \mu_n^{'} (s)} = \frac{\sum_{x \neq s} \rho_{n+1} (x)}{\sum_{x \neq s} \mu_{n+1} (x)}
\end{equation*}
\end{proof}

\exploration*
\begin{proof}
When $p < 1$, the exploration bonus decreases which in turn lower the probability that agent is guaranteed to act optimally. \\
Concretely in MBIE-EB proof the bonus is crucial to show that the \emph{optimism in the face of uncertainty} behavior is verified at all timesteps. We review here shallowly how using a bonus  $\sqrt{p} \beta (N_n (s,a))^{-1/2}$ impacts this result, for an in depth review we refer to the original work of \citet{strehl2008analysis}. \\
For some state-action pairs $(s, a)$ consider the first $k \le m$ experiences of $(s,a)$ by the agent and let $X_1, ..., X_k$ be the k random variables defined by:
$ X_i \coloneqq r_i + \gamma V^{*} (s_i)$. Where $r_i$ and $s_i$ are the $i$-th reward received and next state after experiencing the pair $(s,a)$
Given $\mathbb{E} [X_i] = Q^* (s,a) $ and $0 \le X_i \le 1 / (1 - \gamma)$, the Hoeffding bound gives:
\begin{equation*} \label{eq:betap}
P \Big[\mathbb{E} [X_1] - \frac{1}{k} \sum_{i=1}^{k} X_i \ge \frac{ \sqrt{p} \beta}{\sqrt{k}} \Big] \le e^{-2 (\sqrt{p} \beta)^2 (1 -\gamma)^2} = \Big( \frac{\delta}{2 | \mathcal{S}_G | | \mathcal{A} | m} \Big)^{p}
\end{equation*}
Which using the union bound allows us to show that:
\begin{equation*} \label{eq:union}
\hat{R}(s,a) + \gamma \sum_{s'} \hat{T} (s, a, s') V^* (s') - Q^* (s,a) \ge - \frac{\sqrt{p} \beta}{\sqrt{k}}
\end{equation*}
holds for all timesteps t and all state-action pairs $(s,a)$ with probability at least $1 - (| \mathcal{S}_G | | \mathcal{A} | m) (\delta / (2 | \mathcal{S}_G  | | \mathcal{A} | m))^{p} $.
For $p < 1$, it is lower than $1 - \delta /2$ and the precision required by MBIE-EB is not achieved.

Likewise, when $p > 1$, the agent can suffer this time from \emph{over-exploration}. To prevent the bonus to modify the reward too much and influence the action gap, $\beta$ and $m$ must verify:
$$ \frac{\beta}{\sqrt{m}} \le \epsilon / 4$$
Which means that a linear increase of $\beta$ has to be compensated by a quadratic increase of $m$. \\
\end{proof}

\approxineq*
\begin{proof}

We have for $\bar{s} \in \mathcal{S}_A, $:
\begin{equation*}
    \forall s, s' \in G(\bar{s}), \,\, (1 - \epsilon)^3 \hN (s', a) \leq \hN (s, a) \leq (1 + \epsilon)^3 \hN (s', a)
\end{equation*}
Summing over all states $s'$ in the aggregation:
\begin{equation*}
    (1 - \epsilon)^3 \, \frac{\rasa}{| G(s) |} \leq \rho_n (s, a) \leq (1 + \epsilon)^3 \, \frac{\rasa}{| G(s) |}
\end{equation*}
\end{proof}
Hence:
\begin{equation*}
\begin{aligned}
f(\bs, a, \epsilon) &= \frac{(1-\epsilon)^3 \rasa (1 - (1+\epsilon)^3 \rapsa / |G(s)|)}{(1+\epsilon)^3 \rapsa - (1-\epsilon)^3 \rasa} \\
&= \frac{|G(s)| (\hna + 1) - (1 + \epsilon)^3(\hNasa + 1)}{|G(s)| (\tfrac{1}{\alpha_\epsilon^3}\hna - \hNasa + (\tfrac{1}{\alpha_\epsilon^3} - 1) \hNasa \hna}
\end{aligned}    
\end{equation*}
Using $\rasa = \hNasa / \hna$, similarly:
\begin{equation*}
\begin{aligned}
g(\bs, a, \epsilon) &= \frac{(1+\epsilon)^3 \rasa (1 - (1-\epsilon)^3 \rapsa / |G(s)|)}{(1-\epsilon)^3 \rapsa - (1+\epsilon)^3 \rasa} \\
&= \frac{|G(s)| (\hna + 1) - (1 - \epsilon)^3 (\hNasa + 1)}{|G(s)| (\alpha_\epsilon^3 \hna - \hNasa - (1 - \alpha_\epsilon^3) \hNasa \hna}
\end{aligned}    
\end{equation*}

\abstoground*

\begin{proof}
Setting $\epsilon = 0$ in the previous result gives:
\begin{equation*}
\begin{aligned}
\hNsa &= \hNasa \frac{| G(s) | (\hna + 1) - (\hNasa + 1) }{| G(s) | (\hna -  \hNasa)} \\
&= \hNasa \Big (1 + \frac{(| G(s) | -1)(\hNasa + 1) }{| G(s) | (\hna -  \hNasa)} \Big)  \\
\end{aligned}
\end{equation*}
\end{proof}

\pseudobound*
\begin{proof}

If there is exists a constant $k$ such that $0 \le \hN^{\text{A}} (\bar{s}) \le \hat{n}^{\text{A}} / k$, we can bound the term:
\begin{equation}
\begin{aligned}
\frac{(| G(s) | -1)(\hat{N}^{\text{A}} (\bar{s}) + 1) }{| G(s) | (\hat{n}^{\text{A}} -  \hat{N}^{\text{A}} (\bar{s}))} &\le \frac{(| G(s) | -1)(\hat{n}^{\text{A}} / k + 1) }{| G(s) | (\hat{n}^{\text{A}} - \hat{n}^{\text{A}} / k)} \\
&\le \frac{\hat{n}^{\text{A}} + k}{k \hat{n}^{\text{A}} - \hat{n}^{\text{A)}}} \\
&\le \frac{2}{k - 1}
\end{aligned}
\end{equation}
\end{proof}

\improvedpseudo*
\begin{proof}
We have the following system of three equations
\begin{align*}
    \rho_n (s,a) &= \frac{\tNsa}{\tn}, \\
    \rpsa &= \frac{\tNsa + 1}{\tn+ |G(s)|}, \\
    \rppsa &= \frac{\tNsa + 2}{\tn+ 2|G(s)|}.
\end{align*}
The first two give
$$
\hNsa = \rsa \frac{1 - \gs \rpsa}{\rpsa - \rsa}
$$
And from the last one
\begin{align*}
\gs &= \frac{\tNsa + 2}{\rppsa (\tn + 2)} \\
\gs &= \frac{\tNsa + 2}{\rppsa \Big( \frac{\tNsa + 1}{\gs \rpsa} + 1 \Big)} \\
\gs \rppsa + (\tNsa + 1) \frac{\rppsa}{\rpsa} &= \tNsa +2 \\
\gs \rppsa &= \tNsa + 2 - \frac{\rppsa}{\rpsa} (\tNsa +1) \\
\gs \rppsa &= \tNsa (1 - \frac{\rppsa}{\rpsa}) + (2- \frac{\rppsa}{\rpsa}) \\
\end{align*}
Then
\begin{align*}
\tNsa &= \rsa \frac{1 - \gs \rpsa}{\rpsa - \rsa} \\
\tNsa &= \rsa \frac{1 - \frac{1}{\rppsa} \Big( (\rpsa - \rpsa) \tNsa + (2 \rpsa - \rppsa) \Big)}{\rpsa - \rsa} \\
\tNsa &= \frac{\rsa}{\rppsa} \frac{\rppsa - \Big( (\rpsa - \rpsa) \tNsa + (2 \rpsa - \rppsa) \Big)}{\rpsa - \rsa} \\
\tNsa &= \frac{\rsa}{\rppsa} \frac{2 (\rppsa - \rpsa) - (\rpsa - \rpsa) \tNsa }{\rpsa - \rsa} \\
\tNsa \Big( 1 &+ \frac{\rsa (\rpsa - \rppsa}{\rppsa (\rpsa - \rsa)} \Big) = \frac{2 \rsa (\rppsa - \rpsa)}{\rppsa (\rpsa - \rsa)} \\
\tNsa &= \frac{2 \rsa (\rppsa - \rpsa)}{\rppsa (\rpsa - \rsa) - \rsa (\rppsa - \rpsa)}.
\end{align*}
By construction we have $\tNsa = \tilde{N}^{\text{A}} (s,a)$ (remember the induced abstraction is exact), besides when $\gs = 1$ which is the case in the induced abstraction we have $\tN = \hN$,  hence $\tilde{N}^{\text{A}}_n = \hat{N}^{\text{A}}_n $
\end{proof}

\end{appendices}
\end{document}